\documentclass[sigconf, nonacm]{acmart} % Set authordraft=true for draft highlighting.
\pdfoutput=1
\renewcommand\footnotetextcopyrightpermission[1]{} % removes footnote with conference information in first column
\setcopyright{none} 

\usepackage[utf8]{inputenc}
\usepackage[T1]{fontenc}
\usepackage{graphicx}
\usepackage{grffile}
\usepackage{longtable}
\usepackage{wrapfig}
\usepackage{rotating}
\usepackage[normalem]{ulem}
\usepackage{textcomp}
\usepackage{amssymb}
\usepackage{capt-of}

\newcommand{\defeq}{\mathrel{\mathop:}=}

\setcitestyle{square}
\usepackage{tikz}             % function graphs
\usepackage{amsthm}           % lemmas, proofs
\usetikzlibrary{matrix}
\usetikzlibrary{arrows}
\usetikzlibrary{positioning}
\usetikzlibrary{backgrounds}
\tikzstyle{int}=[draw, fill=blue!20, minimum size=2em]
\tikzstyle{init} = [pin edge={to-,thin,black}]
\newcommand{\icol}[1]{\left[ \begin{smallmatrix}#1\end{smallmatrix}\right ]}


\date{\today}
\title{Finite Group Equivariant Neural Networks for Games}

\author{Oisín Carroll}
\affiliation{%
  \institution{Trinity College Dublin}
  \city{Dublin}
  \country{Ireland}
}
\email{carroloi@tcd.ie}

\author{Joeran Beel}
\affiliation{%
  \institution{University of Siegen}
  \city{Siegen}
  \country{Germany}
}
\email{joeran.beel@uni-siegen.de}

\begin{document}

\begin{abstract}
Games such as go, chess and checkers have multiple equivalent game states, i.e. multiple board positions where symmetrical and opposite moves should be made. These equivalences are not exploited by current state of the art neural agents which instead must relearn similar information, thereby wasting computing time. Group equivariant CNNs in existing work create networks which can exploit symmetries to improve learning, however, they lack the expressiveness to correctly reflect the move embeddings necessary for games. We introduce Finite Group Neural Networks (FGNNs), a method for creating agents with an innate understanding of these board positions. FGNNs are shown to improve the performance of networks playing checkers (draughts), and can be easily adapted to other games and learning problems. Additionally, FGNNs can be created from existing network architectures. These include, for the first time, those with skip connections and arbitrary layer types. We demonstrate that an equivariant version of U-Net (FGNN-U-Net) outperforms the unmodified network in image segmentation.

\end{abstract}

\maketitle

\section{Introduction}
\label{sec:org79f4b73}
The leading computer algorithms for playing many board games are deep-learning based. Google's DeepMind created the first Go program able to beat a world champion; AlphaGo \cite{Silver2016} and their subsequent papers discuss AlphaZero; a generic algorithm that was trained to become a top go, shogi or chess engine \cite{DBLP:journals/corr/abs-1712-01815}. LeelaChessZero --- a community lead effort to replicate AlphaZero for chess (based on LeelaZero, a go program) recently won the computer chess world championship\footnote{https://www.chess.com/news/view/lc0-wins-computer-chess-championship-makes-history}. The primary innovations in these papers are the methods used to train the networks stably through self-play; loss functions, genomes and randomness, as well as a much improved weighted Monte Carlo tree search.

\begin{figure}[htbp]
\centering
\includegraphics[width=0.38\textwidth]{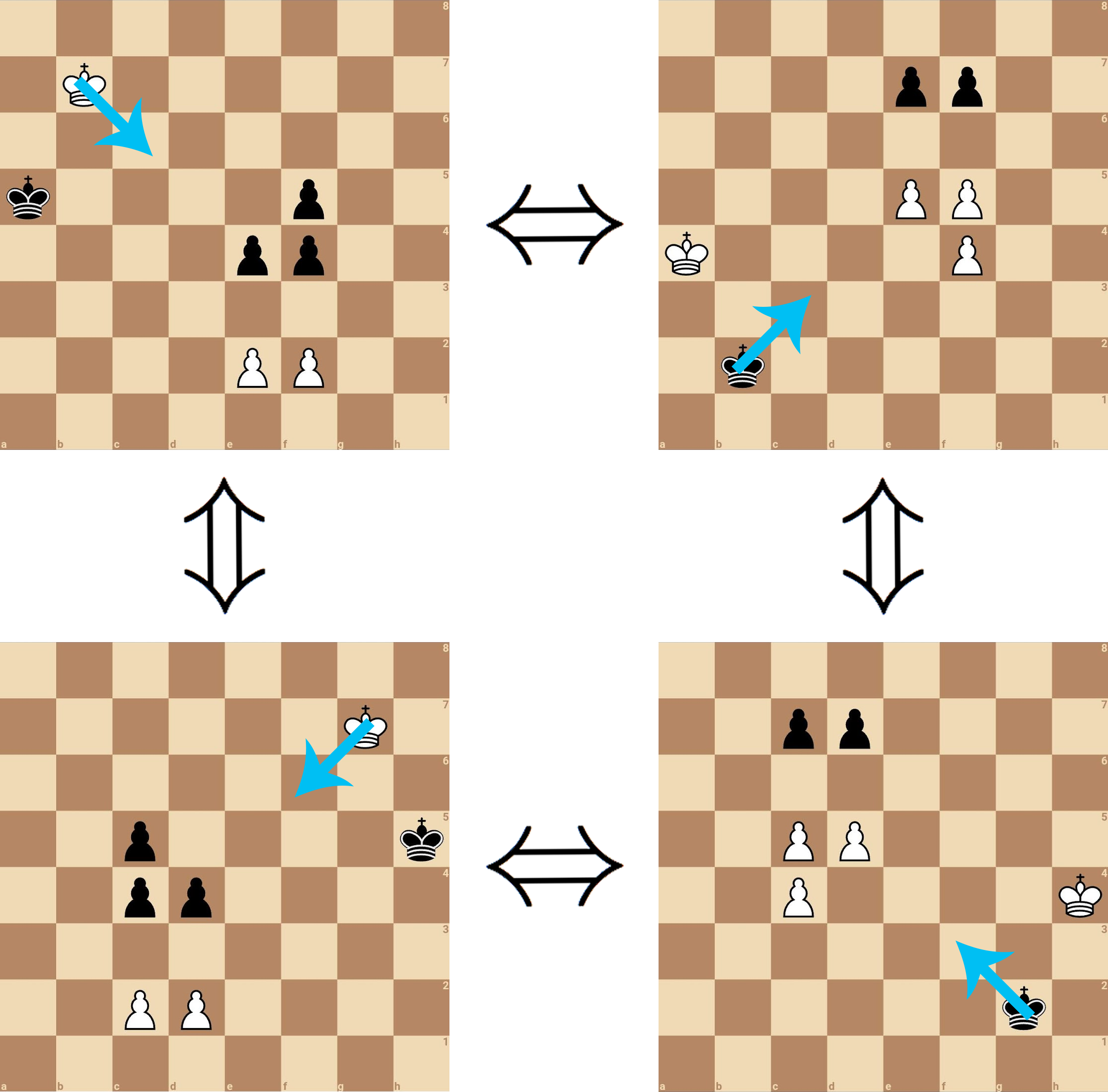}
\caption{\label{fig:orge156d2e}Chess positions with white to play upwards (left) and black to play downwards (right). These are all equivalent through reflecting left-to-right and/or swapping colours, and the only drawing move is shown in blue on all boards.}
\end{figure}

The neural network architectures used at the core of these algorithms are similar however; variations on standard Convolutional Neural Networks (CNNs). They are not able to understand symmetries of the game rules or board and because of this, relearn similar information multiple times. For example,
solving for the best move in any of the 4 boards shown in Figure \ref{fig:orge156d2e} is an equivalent problem; knowing the move in any one means you should play the symmetrically-equivalent move in the rest. More formally,
for some network predicting a move from a board-state \(N: X \rightarrow Y\) and an operation \(g\) which reflects the board and \(g'\) which reflects a move, \(N(gx) = g'N(x)\) for all \(x \in X\). This property is \emph{equivariance}; some transformation of the input leads to an equivalent transformation of the output. In other cases where \(g'\) is the identity function it is \emph{invariance}; where the output is unaffected by some transformation of the input.

If a network was equivariant over these equivalent board positions, the training time theoretically could be reduced by 50-75\%, or the network could potentially achieve higher performance in the same training time. If this equivariance was achieved through weight-sharing, model sizes may also be reduced by a similar amount, along with over-fitting of those networks.

In non-gaming scenarios like image classification and segmentation, research into neural networks which are invariant or equivariant (respectively) to some property of their input is a rapidly expanding field, with many influential works being published in recent years. These include invariances over sets \cite{qi2017pointnet,zaheer2017deep,vinyals2015order,bloem2019probabilistic}; where the order of the input sequence is ignored, graphs \cite{DBLP:journals/corr/KipfW16,DBLP:journals/corr/GilmerSRVD17,maron2018invariant,bloem2019probabilistic}; where perturbations of the input nodes and edges are ignored, and spatial equivariance to a variety of group operations.
These are generally CNN-based and can be divided into networks which are equivariant over finite groups such as rotations of 90\textdegree{}
\cite{pmlr-v48-cohenc16,dieleman2016exploiting,winkels20183d,worrall2018cubenet,romero2020attentive}
and methods which are approximately equivariant over continuous groups, such as affine transformations \cite{gens2014deep,weiler2018learning,bekkers2018roto,lenssen2018group,DBLP:journals/corr/abs-1801-10130,cohen2019general,smets2020pde}.

Games however require a different kind of equivariance. Reflecting the input board state should result in the network predicting the symmetrically opposite move which, in many cases, is not the same as reflecting the output tensor. \emph{AlphaZero} \cite{DBLP:journals/corr/abs-1712-01815} for chess encodes moves in the output where each layer corresponds to a specific direction (North/North-East/East/\ldots{}) and distance to move the piece. Reflecting this move tensor won't correctly reverse the move direction.
Existing research into group-equivariant CNNs cannot easily be extended to handle this case.

Additionally, current research is not shown to work for networks with skip connections, which are an indispensable component of many popular neural network architectures \cite{he2015residual,he2016identity,srivastava2015training,huang2017densely,DBLP:journals/corr/RonnebergerFB15}, and CNN based approaches in existing work also do not extend to networks with fully connected layers, as are used in AlphaGo and AlphaZero.

We introduce Finite Group Neural Networks (FGNNs), which are equivariant over arbitrary finite groups. FGNNs can be used to derive equivariant versions of existing neural network architectures including, for the first time, those with skip connections and arbitrary layer types. Importantly, FGNNs can represent the reflection of move embeddings necessary for games.
% FGNNs can be viewed as a generalization of the work of Dieleman et Al. \cite{dieleman2016exploiting} to arbitrary groups and with an additional proof for skip connections, although the approach differs. %JB dieleman2016exploiting should be mentioned earlier 
We demonstrate that FGNNs reduce over-fitting and improve performance when compared to equivalent networks without this equivariance at playing checkers (draughts). This performance increase is present across a wide variety of model sizes, and the methods presented are general enough to be used in a variety of board games including chess, go and shogi. We conclude with FGNN-based implementations of a popular architecture utilizing skip connections; U-Net \cite{DBLP:journals/corr/RonnebergerFB15}, which are equivariant over several finite groups. We show these outperform the unmodified network, and demonstrate equivalent performance with 4-8 times fewer weights in the task of image segmentation.

\section{Related Work}
Many commonly used CNN architectures for image classification and segmentation, such as ResNet \cite{he2015residual} and DenseNet \cite{huang2017densely}, are approximately invariant or equivariant to small translations of the input \cite{zhang2019making}. This is a form of symmetry, and it is natural then to consider extensions of this symmetry to rotations and reflections of the input.

G-CNNs provide equivariance over arbitrary finite transformation groups by modifying the convolution operation \cite{pmlr-v48-cohenc16}, called a G-Convolution, and provide the basis for many of the recent works in this area.
G-CNNs are used to solve 3D problems over voxel data \cite{winkels20183d,worrall2018cubenet}, and are extended to support attention \cite{romero2020attentive}.
G-Convolutions are utilized to create equivariance over the SE(2) group; consisting of all rotations and translations \cite{weiler2018learning,bekkers2018roto,smets2020pde},
general Lie groups using capsule networks \cite{lenssen2018group}.
Other works utilize G-Convolutions to create translation invariance to input in different homogeneous spaces, such as across the surface of a sphere
\cite{DBLP:journals/corr/abs-1801-10130,cohen2019general}.

Dieleman et al. \cite{dieleman2016exploiting} create a CNN equivariant to cyclic symmetries --- rotations of 90\textdegree{}, by
applying each layer of a network multiple times to different transformations of the input. Their approach doesn't place restrictions on the layer types of the network, however when applied to networks with only convolutional layers, they create equivalent models to G-CNNs. They briefly comment on issues with extending the approach to arbitrary groups.

\section{Finite Group Neural Networks}
\label{sec:org5ee5c9d}
We introduce \emph{Finite Group Neural Networks} or FGNNs. The core idea of FGNNs is that instead of each layer in the resulting network being equivariant to some target group \(G\), it is instead equivariant to some embedding of that group. We define an operation \(T_g\) for each \(g \in G\) and enforce each layer's equivariance to this instead. This equivariance can be created by applying each layer several times to specific transformations of the input, and concatenating the result. This method can be used to create equivariant FGNNs from existing network architectures regardless of their component layers.

Our method can be viewed as an extension of the work of Dieleman et al. \cite{dieleman2016exploiting} to arbitrary groups. Our $\mathit{Lift}$ and $\mathit{Drop}$ layers are analogous to their $\mathit{Slice}$ and $\mathit{Pool}$ layers when our method is applied to the same cyclic symmetry group. Our method and proofs apply to arbitrary finite groups however, and we add an additional layer type for merging, allowing our networks to contain skip connections. Additionally, the method we use for reasoning about the network's equivariance and invariance is easily extensible to, among other uses, express the reflection of move embeddings necessary for games.

\subsection{Equivariance to Horizontal Flips}

In order to aid description, this section gives a practical example of an FGNN derived from a simple network which is equivariance over horizontal reflections. The formalization of this which applies to arbitrary layers, network architectures and groups follows.

A neural network without recurrent layers or memory is a pure function \(N(X) = Y\). It can be written as the composition of $k$ separate layers \(f_i, i \in [1..k]\), which are linearly composed for now.
\begin{center}
\begin{tikzpicture}[node distance=2.5cm,auto,>=latex']
    \node [int] (a) {$X_0$};
    \node [int] (b) [right of=a] {$X_1$};
    \node [int] (c) [right of=b] {$Y$};
    \path[->] (a) edge node {$f_0$} (b);
    \path[->] (b) edge node {$f_1$} (c);
\end{tikzpicture}
\end{center}

We add the restriction that the input tensor \(X\) must have an even number of 'layers', and write it as two stacked subtensors, or 'slices' of equal size.
$$
X = \icol{ X^{(1)} \\ X^{(2)} }
$$
In order to modify this network, we replace all functions \(f_i\) with their modified version \(f_i'\). \(g\) is a matrix which horizontally reflects the input. Note the reordering of the subtensors in the second call of \(f_i\).
$$
f_i'\left(\icol{ X^{(1)} \\ X^{(2)} }\right) = \begin{bmatrix}
            f_i\left(\icol{ X^{(1)} \\ X^{(2)} }\right) \\
            gf_i\left(g\icol{ X^{(2)} \\ X^{(1)} }\right)
        \end{bmatrix}
$$
Then the resulting equivariant network \(N'\) can be written as:
\begin{center}
\begin{tikzpicture}[node distance=1.9cm,auto,>=latex']
    \node [int] (in) {$X_0$};
    \node [int] (a) [right of=in, node distance=1.6cm]{$\icol{ X_0 \\ X_0 }$};
    \node [int] (b) [right of=a] {$\icol{ X_1^{(0)} \\ X_1^{(1)} }$};
    \node [int] (c) [right of=b] {$\icol{ Y^{(0)} \\ Y^{(1)} }$};
    \node [int] (out) [right of=c] {$\substack{Y^{(0)}\\ + Y^{(1)}}$};

    \path[->] (in) edge node {$\mathit{Lift}$} (a);
    \path[->] (a) edge node {$f_0'$} (b);
    \path[->] (b) edge node {$f_1'$} (c);
    \path[->] (c) edge node {$\mathit{Drop}$} (out);
\end{tikzpicture}
\end{center}

Along with switching out the layers, we add two more here. The \(\mathit{Lift}\) layer at the start simply duplicates the input. The \(\mathit{Drop}\) layer at the end adds the two slices of $X$.
\begin{align}
\mathit{Lift}(X) &= \icol{X\\X} \\
\mathit{Drop}\left(\icol{X^{(0)}\\X^{(1)}}\right) &= X^{(0)}+X^{(1)}
\end{align}
Note that the dimensionality of the component layers may need to be changed, i.e. for convolutional layers the number of features should be halved in all layers except the last to retain the same model size and maintain layer inter-connectivity.

In order to be reflection equivariant, the following property should hold for all inputs:
$$
N'(X) = Y \implies N'(gX) = gY = gN'(X)
$$
This can be verified by defining an operation \(T : X \rightarrow X\) which both applies \(g\) and reverses the order of the component tensors. Since \(g\) reflects the tensors, it can be applied equally to each sub-tensor \(X^{(i)}\).
$$
T\icol{ X^{(1)} \\ X^{(2)} }
= g\icol { X^{(2)} \\ X^{(1)} }
= \icol { gX^{(2)} \\ gX^{(1)} }
$$
We can then show that applying \(g\) the input leads to the following network values after each layer is applied, with the result correctly showing equivariance.

\begin{center}
\begin{tikzpicture}[node distance=1.8cm,auto,>=latex']
    \node [int] (in) {$gX_0$};
    \node [int] (a) [right of=in, node distance=1.6cm]{$T\icol{ X_0 \\ X_0 }$};
    \node [int] (b) [right of=a] {$T\icol{ X_1^{(0)} \\ X_1^{(1)} }$};
    \node [int] (c) [right of=b] {$T\icol{ Y^{(0)} \\ Y^{(1)} }$};
    \node [int] (out) [right of=c, node distance=2.1cm] {$g \left(\substack{Y^{(0)}\\ + Y^{(1)}}\right)$};

    \path[->] (in) edge node {$\mathit{Lift}$} (a);
    \path[->] (a) edge node {$f_0'$} (b);
    \path[->] (b) edge node {$f_1'$} (c);
    \path[->] (c) edge node {$\mathit{Drop}$} (out);
\end{tikzpicture}
\end{center}

The necessary properties (below) of \(\mathit{Lift}\) and \(\mathit{Drop}\) layers can both be verified from the definitions. Reordering identical slices and reordering tensors before adding them, respectively, have no effect.
\begin{align}
\mathit{Lift}(gX) &= T\mathit{Lift}(X)\label{lft-def}\\
\mathit{Drop}(TX) &= g\mathit{Drop}(X)\label{drp-def}
\end{align}
Demonstrating the commutativity of \(T\) over \(f_i'\) is slightly more challenging.
\begin{align}
f_i'\left(T\icol{ X^{(1)} \\ X^{(2)} }\right)
= \begin{bmatrix}
    f_i\left(T\icol{ X^{(1)} \\ X^{(2)} }\right) \\
    gf_i\left(gT\icol{ X^{(2)} \\ X^{(1)} }\right) \\
\end{bmatrix}
= \begin{bmatrix}
    f_i\left(g\icol{ X^{(2)} \\ X^{(1)} }\right) \\
    gf_i\left(g^2 \icol{ X^{(1)} \\ X^{(2)} }\right) \\
\end{bmatrix}
\end{align}
Then using the fact that \(g\) is its own inverse, i.e. reflecting something twice has no effect.
\begin{align}
&= \begin{bmatrix}
    f_i\left(g\icol{ X^{(2)} \\ X^{(1)} }\right) \\
    gf_i\left(\icol{ X^{(1)} \\ X^{(2)} }\right) \\
\end{bmatrix}
= Tg^{-1}\begin{bmatrix}
    gf_i\left(\icol{ X^{(1)} \\ X^{(2)} }\right) \\
    f_i\left(g\icol{ X^{(2)} \\ X^{(1)} }\right) \\
\end{bmatrix}\nonumber\\ 
&=
T\begin{bmatrix}
    f_i\left(\icol{ X^{(1)} \\ X^{(2)} }\right) \\
    gf_i\left(g\icol{ X^{(2)} \\ X^{(1)} }\right) \\
\end{bmatrix}
= Tf_i'\left(\icol{ X^{(1)} \\ X^{(2)} }\right)
\end{align}
Hence each \(f_i\) commutes with \(T\). In combination with the properties of the lift and drop layers, the full network is equivariant to flips.

\subsection{T-Equivariance}
\label{sec:org3565be3}
We formalize the method which created the equivariant network in the previous section, and show how it can create networks which are equivariant over arbitrary finite groups (with minor restrictions), and from networks with skip connections. There are also more efficient implementations for some layer types such as pooling.

For this the input tensor of each layer, denoted \(X\), must be evenly divisible into \(|G|\) slices (this is later enforced by the architecture). Additionally the group elements must commute with slices of the tensor. This is true for rotations and reflections since, for example, reflecting a tensor simply reflects each slice individually.
\begin{equation}
gX = \icol{gX^{(1)} \\ gX^{(2)} \\ \vdots}\label{mat-commute}
\end{equation}

\(T_g : X \rightarrow X\) is defined for each element of the group \(g \in G\), and consists of splitting the input tensor into \(|G|\) slices, reordering them before concatenating them, and applying \(g\). We denote this reordering as \(R_g\).
Since the reordering of slices must commute with the group elements \(G\):
\begin{equation}
T_g \defeq R_g \circ g = g \circ R_g\label{r-commute}
\end{equation}

\theoremstyle{definition}
\begin{definition}{T-Equivariance}
A function is T-Equivariant if it commutes with all of the resulting operations.
$$
\forall g \in G : T_g \circ f = f \circ T_g
$$\end{definition}
\begin{figure}[htbp]
\centering
\includegraphics[width=0.32\textwidth]{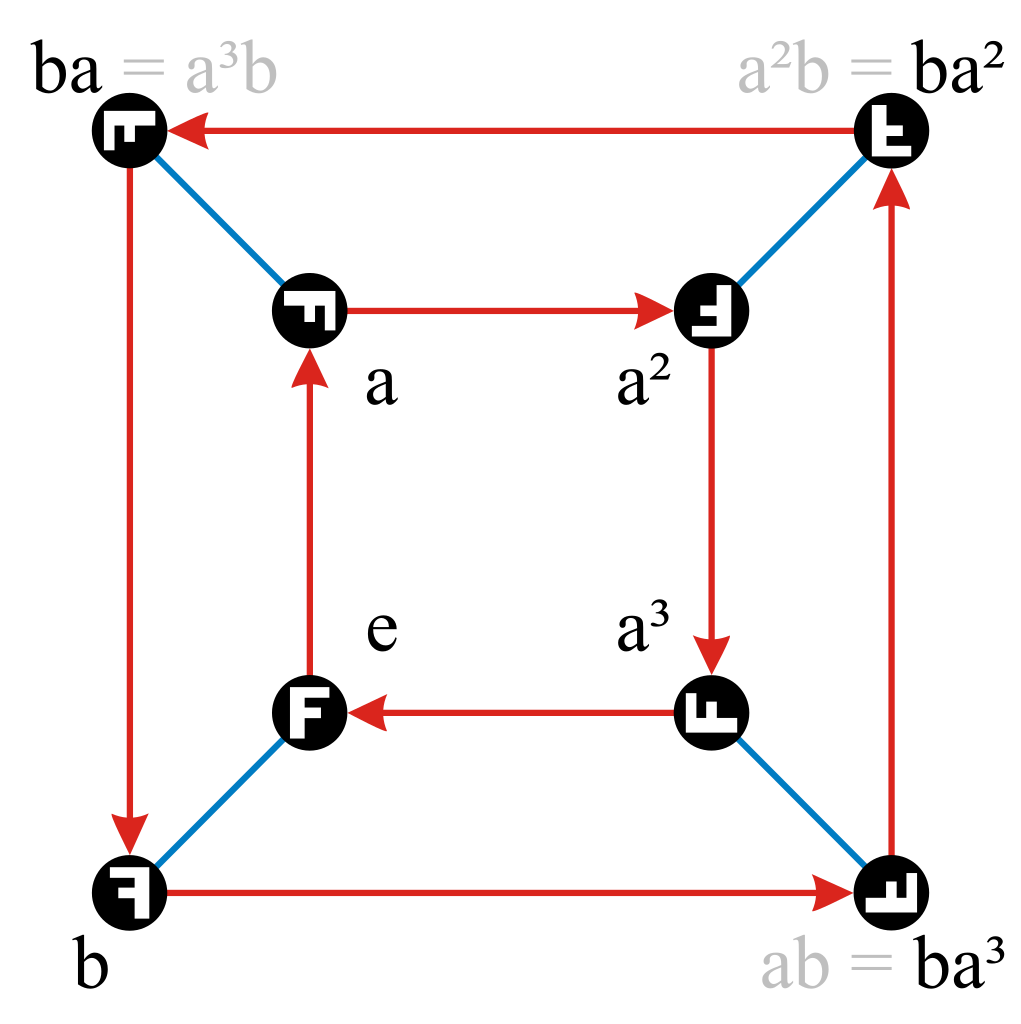}
\caption{\label{fig:org50fec45}The cayley graph of the dihedral group of the square; \(D_8\). This defines two generating functions; \(a\) (red) and \(b\) (blue), for 90\textdegree{} rotation and horiztontal reflection respectively. The nodes then represent all elements of the group/ all possible transformations of an image.}
\end{figure}

The exact reordering of pieces is defined similarly to a cayley graph representation of the group, which has a node for each element. We assign each of the \(|G|\) slices of the tensor arbitrarily to nodes in the graph then permute them according to way that \(g\) maps elements onto each-other.  For example, the generating element \(a\) is shown in Figure \ref{fig:org50fec45}, \(R_a\) permutes the slices according to the red arrows. Alternatively, this can be viewed as representing \(G\) as a subgroup of an equivalently sized permutation group (\(S_n\)). Formally:

Let \(m : G \rightarrow \mathbb{N}\) be a function which arbitrarily assigns each \(g \in G\) a unique index \(\in [1..|G|]\). Let \([g_1, g_2, \dots]\) be the resulting mapping applied to the group.
Note that \(m(g_i) = i\) by definition.

Consider the input tensor \(X\) as divided into \(|G|\) sub-tensor 'slices' denoted \([X^{(1)}, X^{(2)} \dots X^{(|G|)}]\).
\(R_s : X \rightarrow X\) is an operation which reorders the input such that in the output tensor \(X'\), \(X'^{(i)} = X^{(m(g_is))}\).

In full, this mapping can be written as:
\begin{align}
    R_s \left( \icol{
           X^{(1)}      \\
           X^{(2)}  \\
           \vdots \\
           X^{(|G|)}  \\
         }
         \right)
     =
        \icol{
           X^{(m(g_1))}      \\
           X^{(m(g_2))}      \\
           \vdots \\
           X^{(m(g_{|G|}))}      \\
         }
\end{align}

\begin{lemma}
\label{tsub-gh}
$T_h T_s = T_{hs}$ for all $h, s \in G$
\end{lemma}
\begin{proof}
From the definition, $R_h$ maps the slice at index $i$ to index $m(g_ih)$.
Applying this definition twice shows $R_hR_s$ maps the slice at index $i$ to index $m(g_{m(g_is)}h)$.

$m(g_i) = i$. Then for any $h \in G$,  $g_{m(h)} = h$. This means $m(g_{m(g_is)}h) = m(g_ihs)$.

Hence, $R_hR_s$ maps the slice at index $i$ to index $m(g_ihs)$, which is the same as $R_{hs}$. Since both reorder the slices equally. $R_hR_s = R_{hs}$

Finally, $T_h T_s = R_hhR_ss = R_h R_s hs = R_{hs}hs = T_{hs}$
\end{proof}

\subsection{T-Equivariant Layers}
\label{sec:orge8370c7}
It is possible to redefine any layer so that it commutes with \(T_s\), for any \(s \in G\).

Given some function \(f\); a layer in our network, we can define its T-equivariant version as \(f'\).
\begin{align}
    f'(X) \defeq \begin{bmatrix}
           f(T_{g_1}X)g_{1}^{-1}     \\
           f(T_{g_2}X)g_{2}^{-1}     \\
           \vdots \\
           f(T_{g_{|G|}}X)g_{|G|}^{-1}     \\
         \end{bmatrix}
\end{align}

\begin{theorem}[T-Equivariance of Simple Functions]
\label{f-equivariance}
$f'(X)$ is T-Equivariant if $g \in G$ commutes with $R_h$.
\end{theorem}
\begin{proof}
In order to prove that $f'$ is equivariant to T, we need to show that $\forall s \in G: T_s(f'(X)) = f'(T_s(X))$.
\begin{align}
    T_s(f'(X)) = T_s &
        \begin{bmatrix}
            f(T_{g_1}X)g_{1}^{-1}     \\
            f(T_{g_2}X)g_{2}^{-1}     \\
            \vdots \\
            f(T_{g_{|G|}}X)g_{|G|}^{-1}     \\
        \end{bmatrix}
    = s \circ R_s
        \begin{bmatrix}
            f(T_{g_1}X)g_{1}^{-1}     \\
            f(T_{g_2}X)g_{2}^{-1}     \\
            \vdots \\
            f(T_{g_{|G|}}X)g_{|G|}^{-1}     \\
            \end{bmatrix}
        \nonumber
\end{align}
Element $i$ in the matrix is $f(T_{g_i}X)g_{i}^{-1}$, and $R_s$ maps each element in the output such that $x'_i = x_{m(g_is)}$. Then the resulting section at index $i$ is $f(T_{g_{m(g_is)}}X)g_{m(g_is)}^{-1}$. Since $g_{m(f)} = f$ this simplifies to $f(T_{g_is}X)(g_{i}s)^{-1}$. Hence:
\begin{align}
    f'(X) = g &
        \begin{bmatrix}
            f(T_{g_1s}X)(g_{1}s)^{-1}  \\
            f(T_{g_2s}X)(g_{2}s)^{-1}  \\
            \vdots \\
            \end{bmatrix}
        \nonumber
\end{align}
From our restriction that the group $G$ must commute with taking slices of the input, $s \in G$ must distribute across slices of the matrix.
Further, from Lemma \ref{tsub-gh} $T_{g_is} = T_{g_i}T_{s}$.
\begin{align}
    T_{s}f'(X) &=
        \begin{bmatrix}
            f(T_{g_1}T_{s}X)g_{1}^{-1}s^{-1}s \\
            f(T_{g_2}T_{s}X)g_{2}^{-1}s^{-1}s \\
            \vdots \\
            \end{bmatrix}
        =
        \begin{bmatrix}
            f(T_{g_1}(T_{s}X))g_{1}^{-1} \\
            f(T_{g_2}(T_{s}X))g_{2}^{-1} \\
            \vdots \\
            \end{bmatrix}
        =
        f'(T_{s}X) \nonumber
\end{align}
Since $f'$ commutes with any arbitrary $T_s, s \in G$, it is T-Equivariant.
\end{proof}

\subsection{Lift \& Drop}
\label{sec:orgebf0d93}
Now that we have the core of our network, we can consider how to enter and remove data from either end. These are layers; \(\mathit{Lift}\) to enter data and \(\mathit{Drop}\) to remove, and are defined by the following properties.
\begin{align}
\mathit{Lift}(gX) &= T_g \mathit{Lift}(X)\\
\mathit{Drop}(T_gX) &= g \mathit{Drop}(X)
\end{align}
For simple transformation groups we use the definitions: \(\mathit{Lift}(X)\) simply stacks \(|G|\) copies of \(X\), and \(\mathit{Drop}(x)\) divides \(X\) into \(|G|\) pieces and sums them.
\begin{equation}
\mathit{Lift}(X) = \icol{X \\ X \\ \vdots}
\end{equation}
\begin{equation}
\mathit{Drop}\left(\icol{X^{(1)} \\ X^{(2)} \\ \vdots}\right) = X^{(1)} + X^{(2)} + \hdots
\end{equation}

The defining properties for these can be verified from properties of \(g\) \eqref{mat-commute} and \(T_g\) \eqref{r-commute}. Importantly this means that other definitions of these layers are possible for different groups, as is important for groups acting on the action spaces of board games.

\subsection{Skip connections}
\label{sec:org9388254}
Skip-connections make the training of very deep networks more stable, and are an indispensable component of a variety of popular neural network architectures \cite{he2015residual,he2016identity,srivastava2015training,huang2017densely,DBLP:journals/corr/RonnebergerFB15}.

\(\mathit{Merge}\) layers which maintain T-Equivariance can be defined simply by splitting and zipping together the input tensors.
\begin{align}
    \mathit{Merge} \left(
    \icol{
           A^{(1)}      \\
           A^{(2)}      \\
           \vdots   \\
           A^{(|G|)}  \\
    },
    \icol{
           B^{(1)}      \\
           B^{(2)}  \\
           \vdots \\
           B^{(|G|)}  \\
    }
    \right)
= \icol{
           A^{(1)}      \\
           B^{(1)}      \\
           A^{(2)}      \\
           B^{(2)}  \\
           \vdots \\
           A^{(|G|)}  \\
           B^{(|G|)}  \\
    }
\end{align}

\begin{lemma}
\label{mrg-commute}
$\mathit{Merge}(TA, TB) = T\circ \mathit{Merge}(A,B)$
\end{lemma}
\begin{proof}
\allowdisplaybreaks
The proof of the merge layer's T-equivarience follows from the definition; the shuffling of the input results in the same shuffling of the output. For all $T_s$, $s \in G$:
\begin{align}
    \mathit{Merge}(T_sA, T_sB) = &\mathit{Merge} \left(
        T_s
        \icol{
            A^{(1)}  \\
            A^{(2)}  \\
            \vdots \\
            A^{(|G|)}  \\
        }
        ,
        T_s
        \icol{
            B^{(1)}      \\
            B^{(2)}  \\
            \vdots \\
            B^{(|G|)}  \\
        }
    \right)\nonumber\\
    =
    &\mathit{Merge} \left(
        \icol{
            A^{(m(g_1s))}s  \\
            A^{(m(g_2s))}s  \\
            \vdots \\
            A^{(m(g_{|G|}s))}s  \\
        }
        ,
        \icol{
            B^{(m(g_1s))}s  \\
            B^{(m(g_2s))}s  \\
            \vdots \\
            B^{(m(g_{|G|}s))}s  \\
        }
    \right)\nonumber
\\&=
    \icol{
            A^{(m(g_1s))}s  \\
            B^{(m(g_1s))}s  \\
            A^{(m(g_2s))}s  \\
            B^{(m(g_2s))}s  \\
            \vdots \\
            A^{(m(g_{|G|}s))}s  \\
            B^{(m(g_{|G|}s))}s  \\
        }
=
    T_s
        \icol{
            A^{1}      \\
            B^{1}      \\
            A^{2}      \\
            B^{2}  \\
            \vdots \\
            A^{|G|}  \\
            B^{|G|}  \\
        }\nonumber\\
        &= T_s\mathit{Merge}(A,B)
\end{align}
\end{proof}

\begin{lemma}
Any skip connection over T-Equivariant functions followed by a $\mathit{Merge}$ layer is itself a T-Equivariant funciton.
\end{lemma}
\begin{proof}

Any skip connection in a network can be visualized as the following graph. Both $f_0'$ and $f_1'$ are T-Equivariant, and may represent the composition of multiple layers and skip connections.

\begin{center}
\begin{tikzpicture}[node distance=2.2cm,auto,>=latex']
    \node [int] (in) {$X_0$};
    \node [int] (a) [right of=in]{$X_1$};
    \node [int] (b) [below of=a, node distance=1cm] {$X_2$};
    \node [int] (mrg) [right of=b, node distance=3cm] {$\mathit{Merge}(X_1, X_2) = Y$};

    \path[->] (in) edge node {$f_0'$} (a);
    \path[->] (in) edge node {$f_1'$} (b);
    \path[->] (a) edge node {$$} (mrg);
    \path[->] (b) edge node {$$} (mrg);
\end{tikzpicture}
\end{center}

$T : T \in \{T_g, g \in G\}$ can commute with this structure, using the definitions of T-Equivariance and Lemma \ref{mrg-commute}. Hence it is T-Equivariant.
\begin{center}
\begin{tikzpicture}[node distance=2.2cm,auto,>=latex']
    \node [int] (in) {$TX_0$};
    \node [int] (a) [right of=in]{$TX_1$};
    \node [int] (b) [below of=a, node distance=1cm] {$TX_2$};
    \node [int] (mrg) [right of=b, node distance=3cm] {$\mathit{Merge}(TX_1, TX_2) = TY$};

    \path[->] (in) edge node {$f_0'$} (a);
    \path[->] (in) edge node {$f_1'$} (b);
    \path[->] (a) edge node {$$} (mrg);
    \path[->] (b) edge node {$$} (mrg);
\end{tikzpicture}
\end{center}
\end{proof}

\subsection{Pooling}
\label{sec:org38a15ee}
The previous definition allows for arbitrary layers to be made T-Equivariant, which includes up or down-scaling layers. However, maximum, minimum, and average pooling layers are symmetric by definition, and applied pointwise. Over the \(D_8\) group or any sub-group this means they are T-Equivariant without modification.

\subsection{Group Equivariance}
\label{sec:orgae5d76f}
Finally, let our network be \(N\); the T-equivariant function \(f'\) composed at either end with a \(\mathit{Lift}\) and a \(\mathit{Drop}\) layer.
$$
N = \mathit{Drop} \circ f' \circ \mathit{Lift}
$$

\begin{theorem}[Group Equivariance]
\label{group-equivariance}

$N = \mathit{Drop} \circ f' \circ \mathit{Lift}$ is equivariant over group $G$.
\end{theorem}
\begin{proof}

Due to the definitions of $\mathit{Lift}$ and $\mathit{Drop}$, $\forall g \in G$:
\begin{align}
N \circ g &= \mathit{Drop} \circ f' \circ \mathit{Lift} \circ g \nonumber \\
         &= \mathit{Drop} \circ f' \circ T_g \circ \mathit{Lift} \nonumber \\
         &= \mathit{Drop} \circ T_g \circ f' \circ \mathit{Lift} \\
         &= g \circ \mathit{Drop} \circ f' \circ \mathit{Lift} \nonumber \\
         &= g \circ N                    \nonumber
\end{align}
\end{proof}

Hence, \(N\) is equivariant over G, since \(g: g \in G\) commutes. All together this allows us then to 'upgrade' all layers in any network to commute with T, and by adding a layer at the start and end the whole network will be equivariant to any chosen finite group.

\subsection{Move Embeddings}
\label{sec:org781a547}
For board-games, the network must output a move, or a policy \(\pi\) over all possible moves. The reflection of this policy tensor or move embedding may not correspond to the symmetrically opposite move.

For example, the \emph{AlphaZero} \cite{DBLP:journals/corr/abs-1712-01815} chess network outputs a policy representing the probability of all possible moves. This is a tensor of size \(8 \times 8 \times 73\). Each of the \(8 \times 8\) positions identifies where to ``pick up'' a piece, while each of the corresponding 73 planes encode how that piece should be moved. The first 56 of which move a piece 1--7 squares along one of the 8 compass directions. The next 8 planes correspond to possible knight moves. Finally, the remaining 9 planes encode possible pawn under-promotions or pawn captures.

\begin{figure}[htbp]
\centering
\includegraphics[width=.9\linewidth]{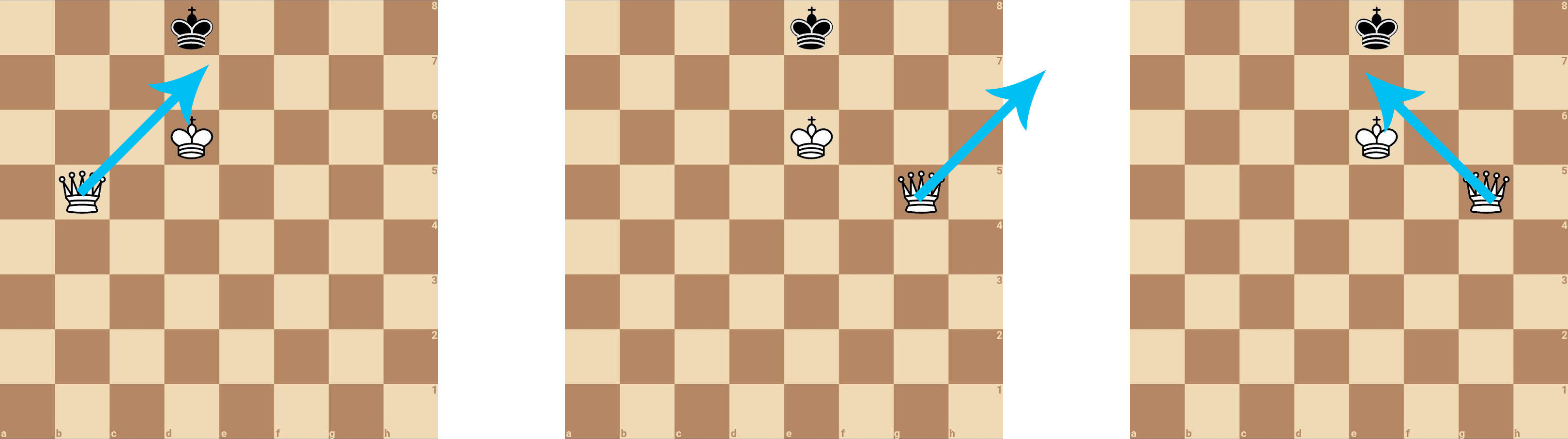}
\caption{\label{fig:orgc65d86a}A move (blue arrow) on a board state (left). The resulting move and board state after naively reflecting both tensors (middle). The correctly reflected board and move (right).}
\end{figure}

Simply reflecting this policy tensor gives the incorrect move, as shown in Figure \ref{fig:orgc65d86a}. Since we can rotate the board and recolour pieces in the input such that white is to play in all cases, we only need the network to be equivariant over horizontal reflections.

We consider the group \(G\) as acting upon this space in such a way that it maps between symmetrically opposite moves. Hence a new \(\mathit{Drop}\) layer must be defined for which its defining property holds.
$$
\mathit{Drop} \circ T_g(X) = g \circ \mathit{Drop}(X) \nonumber
$$

This can be done by splitting the move space into subtensors relating to symmetrical; \(S_0\) moves, non-symmetrical moves moving left; \(X_0\) and non-symmetrical moves moving right; \(X_1\). If \(g\) horizontally reflects the board
$$
X = \icol{S_0 \\ X_0 \\ S_0' \\ X_1}
$$
$$
\mathit{Drop}'\left(\icol{S_0 \\ X_0 \\ S_0' \\ X_1}\right) = \icol{\frac{S_0 + S_0'}{2} \\ X_0 \\ X_1}
$$

Then equivariance over horizontal reflections can be seen by:
$$
\mathit{Drop}'\left(T_g\icol{S_0 \\ X_0 \\ S_0' \\ X_1}\right) =
\mathit{Drop}'\left(\icol{gS_0' \\ gX_1 \\ gS_0 \\ gX_0}\right) =
g\icol{\frac{S_0 + S_0'}{2} \\ X_1 \\ X_0}
$$

Since this resulting network has \(X_0\) and \(X_1\) switched --- each symmetrically relating to opposite moves, the move embedding is correctly flipped. This \(\mathit{Drop}'\) layer then, when used in conjunction with any FGNN, will correctly reflect the move embeddings when the input tensor is reflected.
\section{Methodology}
\label{sec:orgd0651cc}
%start with how you implemented the method
We implement FGNNs in Tensorflow 2.0, and our U-Net implementation is based on an open-sourced example\footnote{https://github.com/zhixuhao/unet}.

To avoid the complexity of specifying each $T_g$ and $R_g$ for every element of the group, our implementation of FGNNs uses the concept of a generating set. This is a subset of the group where any element of the group can be created only from the elements in the generating set.
For example, the \(D_8\) group, consisting of all symmetries of a square, can be generated using 2 elements; a horizontal reflection and a rotation of 90\textdegree{}. 

Groups can be defined by specifying the generating elements as tensorflow functions, and how they each one permutes the tensor slices; a list representing \(R_g\). 
This group object can be passed to \(\mathit{Lift}\), \(\mathit{Drop}\), \(\mathit{Merge}\) and any other layer. The resulting network will be equivariant over that group, provided the commutativity rules hold. This source code is available at \url{https://github.com/FGNN-Author/FGNN}.

\subsection{Datasets}
In order to enable quick iteration of ideas and smaller neural networks to be tested, we choose an equivalently smaller/ simpler game. Checkers (or draughts) is a 2 player game played on an \(8 \times 8\) board. It is a solved game \cite{schaeffer2007checkers}, and has a relatively small number of possible states (\(\sim 5\times 10^{21}\) \cite{schaeffer2007checkers}) compared to chess (\(\sim10^{43}\) \cite{shannon1950xxii}), or go (\(\sim 2 \times 10^{170}\) \cite{tromp2016number}). Still it has many of the properties which makes creating equivariant networks for board-games challenging; pieces which 'move' resulting in more complicated move or policy embeddings, and multiple piece types (after promotion).

We use a dataset of 22 thousand tournament games compiled by the Open Checker Archive\footnote{Open Checker Archive: http://www.fierz.ch/download.php} for training and evaluating the networks. When a move consists of multiple captures, called jumps, these are added to the dataset as multiple board positions and moves, each consisting of a single jump. The resulting training and test sets together contain almost 1 million board states along with the next move made from that state. %Training a network on this, then, is effectively teaching it to predict the moves made by high level checkers players.

\begin{figure}[htbp]
\centering
\includegraphics[width=.9\linewidth]{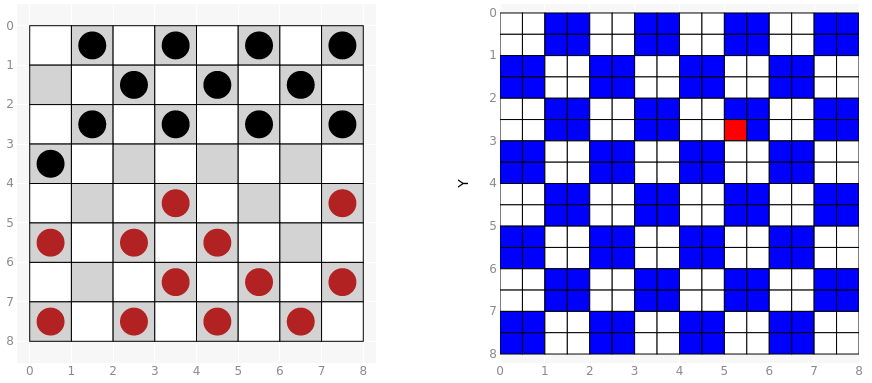}
\caption{\label{fig:orgbd7c1af}The current board state (left) is given as input to the network and the next move to be played, in red (right), is predicted by the network.}
\end{figure}

The board state is given to the network as a single $8 \times 8$ matrix, with values for each square. If the square is empty this value is 0, otherwise for regular pieces it is -1 or 1 for black or red respectively. Squares containing kings similarly are -3 or 3 for black or red. The moves made are one-hot encoded into an vector of size 128, with a space for each square on the board (32) times each of the directions (4) that a piece may move from that square. These are shown in blue/red in Figure \ref{fig:orgbd7c1af}.

Additionally, we test models' performance on the EM segmentation challenge\footnote{http://brainiac2.mit.edu/isbi\_challenge/}. The benchmark was started at ISBI 2012 and is still open for new contributions. The training data is a set of 30 images (512x512 pixels), which are from serial section transmission electron microscopy of the Drosophila first instar larva ventral nerve cord (VNC).

\subsection{FGNNs For Checkers}

%As discussed previously, an issue with using existing methods for equivariant networks for games is their inability to 'mirror' more complex representations such as moves. One strength of our method is that it is easier to reason about. Transformations of the input (\(g\), where \(g \in G\)) transform the output prior to the \(\mathit{Drop}\) layer by \(T_g\). Defining \(\mathit{Lift}\) and \(\mathit{Drop}\) as defined previously allows for networks to be equivariant to \(g\) when both input and output are acting in the same space.

There are effectively four ways the that the same board state can occur in checkers. Two of these are for each player, times two reflections along the horizontal axis of the board. We rotate the board so that the current player is always playing upwards, and 'recolour' them so the current player is black. Because of this, the architecture only needs to only be equivariant over horizontal reflections of the game board.

For checkers, we choose to output a tensor which is \(8 \times 8 \times 4\). Similar to \emph{AlphaZero} each of the \(8 \times 8\) vectors correspond to a starting square for a piece, and the 4 layers represent the 4 directions to move a piece from that square. Both single-square moves and captures; moving 2 squares, are represented by the same layer in the output.

By choosing the order of these 4 layers to be the compass directions: NE SE NW SW, there is a simple method of creating an equivariant architecture. Reflecting a move involves swapping any NE move with NW, and SE with SW, which is the same as swapping the corresponding layers in the output. Finally, the full tensor can simply be horizontally reflected. For example, a move to the NE at (0,0) becomes a move to the NW at (7,0). This way to 'reflect' a move is the exact definition for \(T_g\) over the group of horizontal reflections, and so architecture can be used by simply defining the \(\mathit{Drop}\) layer to be the identity function.

%\subsubsection{Model sizes \& Training}
%\label{sec:orgb999e46}
The state-of-the-art neural architectures for games such as chess and go are highly optimized methods with lots of specific implementation details. Our goal however was to generally evaluate if equivariant networks could be useful in this space. With this in mind, we choose to use a simple CNN architecture with only convolutional layers for both our FGNN implementation and baseline. CNNs can be easily scaled by small amounts adding or removing filters, which allows us to evaluate how equivariance may effectively scale with larger or smaller models; it may only improve the expressive capacity of small networks, or reduce over-fitting in larger ones.

We use CNNs with 10 layers, where each layer has the same number of filters, has a $3 \times 3$ convolution, with \emph{ReLU} activation \cite{glorot2011deep}, and is zero-padded. The number of filters is varied to give a different number of trainable weights in a variety of models. In the final layer, models mask out the 32 squares which correspond to the reachable squares in a checkers game, and flatten the values into a single vector of size 128. A final softmax layer ensures this vector is normalized.

Our FGNN-CNN models are identical to the CNN baseline other than having 2 filters rather than 4 in the last convolutional layer. The size of the output from an FGNN layer with a certain number of trainable weights isn't the same as it's equivalent non-equivariant version, meaning the sizes of the FGNN networks don't match exactly with the baselines. All models are trained for 50 epochs with a categorical cross-entropy loss.

\subsection{FGNNs For Biomedical Segmentation}
To see how this approach may extend to other domains an network architectures, we test FGNN variants of U-Net \cite{DBLP:journals/corr/RonnebergerFB15} (FGNN-U-Net) on the ISBI 2012 challenge dataset.

We create variants of FGNN-U-Net which are equivariant over different groups. These include horizontal reflections, vertical and horizontal reflections, and rotations of 90\textdegree{} and reflections (the \(D_8\) group). By scaling up/down the number of filters in all the layers in the network, we also create variants with differing numbers of trainable weights. The number of filters in each layer of U-Net can only easily be scaled by factors of two, doubling or halving the number of weights in the network. Due to this and the FGNN architecture it isn't always possible to create equivalently sized networks with equivariances to different groups.

\section{Results}
\label{sec:org9c067e1}

\subsection{FGNN-CNN (Checkers)}
\label{sec:orgb03d2ae}
\begin{figure}[htbp]
\centering
\includegraphics[width=0.5 \textwidth]{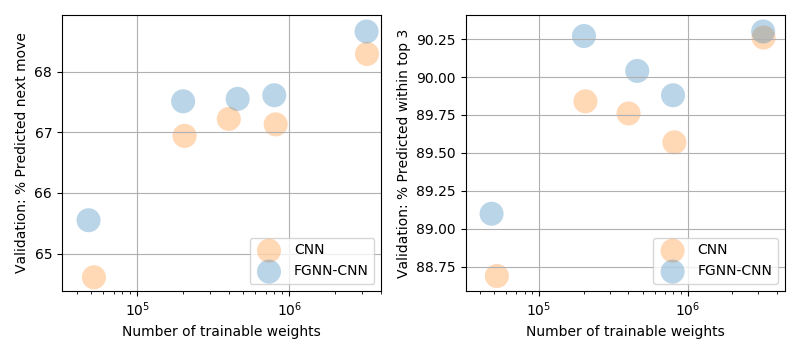}
\caption{\label{fig:org28a5708}The accuracy (left) and top-3 accuracy (right) of equivariant and baseline networks of different sizes on unseen board states.}
\end{figure}

Figure \ref{fig:org28a5708} compares the various models' performance at predicting the next move to be made on unseen board positions. 
FGNN-CNN models predict the next move and have the next move among the top-3 moves predicted by the network more often than equivalent CNN architectures with similar numbers of trainable weights. This performance increase is present across a wide variety of different model sizes, from models which underfit to models which significantly overfit the data.

\begin{figure}[htbp]
\centering
\includegraphics[width=0.5 \textwidth]{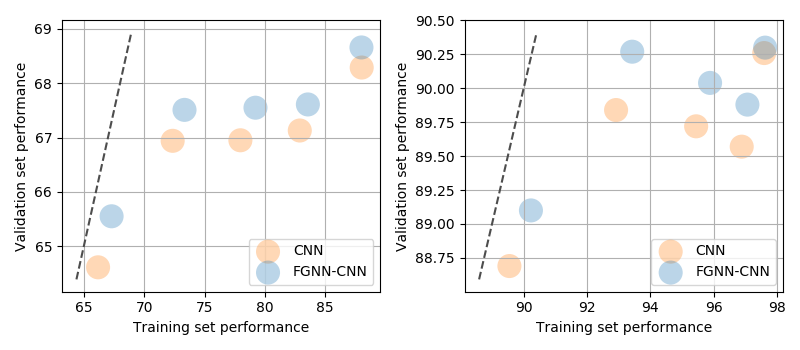}
\caption{\label{fig:orgffaac19}The models' performance on the training set vs the test set. The dashed line is where both are equal, and hence no over-fitting occurred.}
\end{figure}

Evaluating the ability of FGNNs method to combat over-fitting is more challenging. Figure \ref{fig:orgffaac19} compares each model's accuracy on the training set (seen examples) against the unseen examples of the validation set. Models which over-fit more will have a larger difference between training and validation set performance, which can be seen as a larger distance to the dashed line.
Here, we demonstrate that FGNN-CNN models are able to marginally reduce over-fitting. For every baseline model there is an equivariant model which outperforms it or is nearer to the dashed line denoting an ideal learner.

\subsection{FGNN-U-Net (Image Segmentation)}

3 FGNN-U-Net variants which are equivariant to different groups are compared to the unmodified U-Net in Figure \ref{fig:orgca40583}. Our equivariant versions of U-Net outperform the baseline, and this performance improvement seems to scale with the larger groups. This shows that FGNNs can effectively be used in networks with skip-connections, a property that isn't possible in existing methods for creating equivariant networks. The fact that equivariance over larger groups further improves performance demonstrates that equivariant networks are likely to show increased performance gains when used on problems which have more symmetries.

\begin{figure}[htbp]
\centering
\includegraphics[width=0.40\textwidth]{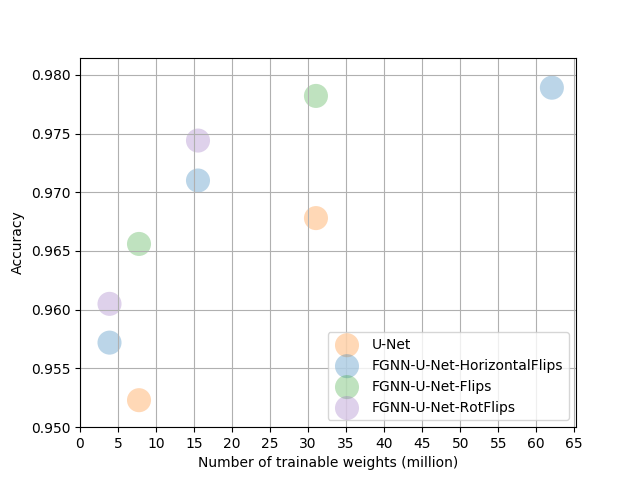}
\caption{\label{fig:orgca40583}The performance of FGNN-U-Net variants of different sizes and symmetry groups compared to the unmodified U-Net network.}
\end{figure}

\section{Conclusion}
\label{sec:org9b754bc}
We have introduced \emph{Finite Group Neural Networks} (FGNNs), neural network architectures which create symmetry equivariant learning algorithms for games.
We demonstrate horizontally equivariant FGNN networks reduce over-fitting and outperform baselines at playing checkers, regardless of the networks' size. Additionally, FGNNs have a strong theoretical foundation and are arguably easier to reason about and extend than existing equivariant architectures. They are the first equivariant architecture which supports skip connections and arbitrary layer types. This is demonstrated by FGNN-U-Net variations, which outperform the unmodified U-Net network in the task of biomedical image segmentation.

The clear continuation of this work is FGNNs' applications to other games. While checkers is a reasonable choice to test a variety of approaches, without any external baselines the performance of models is hard to assess. In the future we would like to create equivariant versions of existing network architectures used for Chess and Go. Go in particular is a game equivariant over the full $D_8$ group so equivariant networks are likely to provide a significant improvement over current approaches.

One of the limitations of the work is that it inherently involves more operations per layer than non-equivariant networks. This slows down training and inference, especially for larger groups, and also scales linearly with the number of elements in the group (which is the same as existing work). However, there may be ways to improve this by using other representations. Often a group can be faithfully represented as, for example, the permutation of a much smaller number of elements. This may provide significant improvements for the speed of FGNNs.

Overall, equivariance is shown to improve the performance of neural networks for playing checkers, and may be prove to be a promising avenue of research in a variety of games. The methods presented may be adapted to create equivariances in existing neural architectures for Chess, Go, and Shogi, as well as a variety of other learning tasks.

\bibliographystyle{unsrt}
\bibliography{refs}

\begin{thebibliography}{10}

\bibitem{Silver2016}
David Silver, Aja Huang, Chris~J. Maddison, Arthur Guez, Laurent Sifre, George
  van~den Driessche, Julian Schrittwieser, Ioannis Antonoglou, Veda
  Panneershelvam, Marc Lanctot, Sander Dieleman, Dominik Grewe, John Nham, Nal
  Kalchbrenner, Ilya Sutskever, Timothy Lillicrap, Madeleine Leach, Koray
  Kavukcuoglu, Thore Graepel, and Demis Hassabis.
\newblock Mastering the game of go with deep neural networks and tree search.
\newblock {\em Nature}, 529(7587):484--489, January 2016.

\bibitem{DBLP:journals/corr/abs-1712-01815}
David Silver, Thomas Hubert, Julian Schrittwieser, Ioannis Antonoglou, Matthew
  Lai, Arthur Guez, Marc Lanctot, Laurent Sifre, Dharshan Kumaran, Thore
  Graepel, Timothy~P. Lillicrap, Karen Simonyan, and Demis Hassabis.
\newblock Mastering chess and shogi by self-play with a general reinforcement
  learning algorithm.
\newblock {\em CoRR}, abs/1712.01815, 2017.

\bibitem{qi2017pointnet}
Charles~R Qi, Hao Su, Kaichun Mo, and Leonidas~J Guibas.
\newblock Pointnet: Deep learning on point sets for 3d classification and
  segmentation.
\newblock In {\em Proceedings of the IEEE conference on computer vision and
  pattern recognition}, pages 652--660, 2017.

\bibitem{zaheer2017deep}
Manzil Zaheer, Satwik Kottur, Siamak Ravanbakhsh, Barnabas Poczos, Russ~R
  Salakhutdinov, and Alexander~J Smola.
\newblock Deep sets.
\newblock In {\em Advances in neural information processing systems}, pages
  3391--3401, 2017.

\bibitem{vinyals2015order}
Oriol Vinyals, Samy Bengio, and Manjunath Kudlur.
\newblock Order matters: Sequence to sequence for sets.
\newblock {\em arXiv preprint arXiv:1511.06391}, 2015.

\bibitem{bloem2019probabilistic}
Benjamin Bloem-Reddy and Yee~Whye Teh.
\newblock Probabilistic symmetry and invariant neural networks.
\newblock {\em arXiv preprint arXiv:1901.06082}, 2019.

\bibitem{DBLP:journals/corr/KipfW16}
Thomas~N. Kipf and Max Welling.
\newblock Semi-supervised classification with graph convolutional networks.
\newblock {\em CoRR}, abs/1609.02907, 2016.

\bibitem{DBLP:journals/corr/GilmerSRVD17}
Justin Gilmer, Samuel~S. Schoenholz, Patrick~F. Riley, Oriol Vinyals, and
  George~E. Dahl.
\newblock Neural message passing for quantum chemistry.
\newblock {\em CoRR}, abs/1704.01212, 2017.

\bibitem{maron2018invariant}
Haggai Maron, Heli Ben-Hamu, Nadav Shamir, and Yaron Lipman.
\newblock Invariant and equivariant graph networks.
\newblock {\em arXiv preprint arXiv:1812.09902}, 2018.

\bibitem{pmlr-v48-cohenc16}
Taco Cohen and Max Welling.
\newblock Group equivariant convolutional networks.
\newblock In Maria~Florina Balcan and Kilian~Q. Weinberger, editors, {\em
  Proceedings of The 33rd International Conference on Machine Learning},
  volume~48 of {\em Proceedings of Machine Learning Research}, pages
  2990--2999, New York, New York, USA, 20--22 Jun 2016. PMLR.

\bibitem{dieleman2016exploiting}
Sander Dieleman, Jeffrey De~Fauw, and Koray Kavukcuoglu.
\newblock Exploiting cyclic symmetry in convolutional neural networks.
\newblock {\em arXiv preprint arXiv:1602.02660}, 2016.

\bibitem{winkels20183d}
Marysia Winkels and Taco~S Cohen.
\newblock 3d g-cnns for pulmonary nodule detection.
\newblock {\em arXiv preprint arXiv:1804.04656}, 2018.

\bibitem{worrall2018cubenet}
Daniel Worrall and Gabriel Brostow.
\newblock Cubenet: Equivariance to 3d rotation and translation.
\newblock In {\em Proceedings of the European Conference on Computer Vision
  (ECCV)}, pages 567--584, 2018.

\bibitem{romero2020attentive}
David~W Romero, Erik~J Bekkers, Jakub~M Tomczak, and Mark Hoogendoorn.
\newblock Attentive group equivariant convolutional networks.
\newblock {\em arXiv preprint arXiv:2002.03830}, 2020.

\bibitem{gens2014deep}
Robert Gens and Pedro~M Domingos.
\newblock Deep symmetry networks.
\newblock In {\em Advances in neural information processing systems}, pages
  2537--2545, 2014.

\bibitem{weiler2018learning}
Maurice Weiler, Fred~A Hamprecht, and Martin Storath.
\newblock Learning steerable filters for rotation equivariant cnns.
\newblock In {\em Proceedings of the IEEE Conference on Computer Vision and
  Pattern Recognition}, pages 849--858, 2018.

\bibitem{bekkers2018roto}
Erik~J Bekkers, Maxime~W Lafarge, Mitko Veta, Koen~AJ Eppenhof, Josien~PW
  Pluim, and Remco Duits.
\newblock Roto-translation covariant convolutional networks for medical image
  analysis.
\newblock In {\em International Conference on Medical Image Computing and
  Computer-Assisted Intervention}, pages 440--448. Springer, 2018.

\bibitem{lenssen2018group}
Jan~Eric Lenssen, Matthias Fey, and Pascal Libuschewski.
\newblock Group equivariant capsule networks.
\newblock In {\em Advances in Neural Information Processing Systems}, pages
  8844--8853, 2018.

\bibitem{DBLP:journals/corr/abs-1801-10130}
Taco~S. Cohen, Mario Geiger, Jonas K{\"{o}}hler, and Max Welling.
\newblock Spherical cnns.
\newblock {\em CoRR}, abs/1801.10130, 2018.

\bibitem{cohen2019general}
Taco~S Cohen, Mario Geiger, and Maurice Weiler.
\newblock A general theory of equivariant cnns on homogeneous spaces.
\newblock In {\em Advances in Neural Information Processing Systems}, pages
  9142--9153, 2019.

\bibitem{smets2020pde}
Bart Smets, Jim Portegies, Erik Bekkers, and Remco Duits.
\newblock Pde-based group equivariant convolutional neural networks.
\newblock {\em arXiv preprint arXiv:2001.09046}, 2020.

\bibitem{he2015residual}
Kaiming He, Xiangyu Zhang, Shaoqing Ren, and Jian Sun.
\newblock Deep residual learning for image recognition.
\newblock {\em CoRR}, abs/1512.03385, 2015.

\bibitem{he2016identity}
Kaiming He, Xiangyu Zhang, Shaoqing Ren, and Jian Sun.
\newblock Identity mappings in deep residual networks.
\newblock {\em CoRR}, abs/1603.05027, 2016.

\bibitem{srivastava2015training}
Rupesh~K Srivastava, Klaus Greff, and J{\"u}rgen Schmidhuber.
\newblock Training very deep networks.
\newblock In {\em Advances in neural information processing systems}, pages
  2377--2385, 2015.

\bibitem{huang2017densely}
Gao Huang, Zhuang Liu, Laurens Van Der~Maaten, and Kilian~Q Weinberger.
\newblock Densely connected convolutional networks.
\newblock In {\em Proceedings of the IEEE conference on computer vision and
  pattern recognition}, pages 4700--4708, 2017.

\bibitem{DBLP:journals/corr/RonnebergerFB15}
Olaf Ronneberger, Philipp Fischer, and Thomas Brox.
\newblock U-net: Convolutional networks for biomedical image segmentation.
\newblock {\em CoRR}, abs/1505.04597, 2015.

\bibitem{zhang2019making}
Richard Zhang.
\newblock Making convolutional networks shift-invariant again.
\newblock {\em arXiv preprint arXiv:1904.11486}, 2019.

\bibitem{schaeffer2007checkers}
Jonathan Schaeffer, Neil Burch, Yngvi Bj{\"o}rnsson, Akihiro Kishimoto, Martin
  M{\"u}ller, Robert Lake, Paul Lu, and Steve Sutphen.
\newblock Checkers is solved.
\newblock {\em science}, 317(5844):1518--1522, 2007.

\bibitem{shannon1950xxii}
Claude~E Shannon.
\newblock Xxii. programming a computer for playing chess.
\newblock {\em The London, Edinburgh, and Dublin Philosophical Magazine and
  Journal of Science}, 41(314):256--275, 1950.

\bibitem{tromp2016number}
John Tromp.
\newblock The number of legal go positions.
\newblock In {\em International Conference on Computers and Games}, pages
  183--190. Springer, 2016.

\bibitem{glorot2011deep}
Xavier Glorot, Antoine Bordes, and Yoshua Bengio.
\newblock Deep sparse rectifier neural networks.
\newblock In {\em Proceedings of the fourteenth international conference on
  artificial intelligence and statistics}, pages 315--323, 2011.

\end{thebibliography}
\end{document}